\documentclass[letterpaper, 10 pt, conference]{Formatting/ieeeconf}  

\IEEEoverridecommandlockouts                              
\usepackage{geometry}
\usepackage{graphics} 
\usepackage{epsfig} 
\usepackage{times} 
\usepackage{amsmath}
\usepackage{amssymb}
\usepackage{comment}
\usepackage{cases}

\usepackage{./Formatting/shortcuts}

\title{\LARGE \bf
 An Input-to-State Stability Perspective on Robust Locomotion
}

\author{Maegan Tucker and Aaron D. Ames
\thanks{This work was supported by Wandercraft and the Zeitlin Family Fund.}
\thanks{Maegan Tucker is with the Department of Mechanical and Civil Engineering, California              Institute of Technology, Pasadena, CA 91125.
        {\tt\small mtucker@caltech.edu}}%
\thanks{Aaron Ames is with both the Department 
        of Mechanical and Civil Engineering and the Department of Computing and Mathematical Sciences, California Institute of Technology, Pasadena, CA 91125.
        {\tt\small ames@caltech.edu}}%
}

\begin{document}

\maketitle
\thispagestyle{empty}
\pagestyle{empty}

\begin{abstract}
Uneven terrain necessarily transforms periodic walking into a non-periodic motion. As such, traditional stability analysis tools no longer adequately capture the ability of a bipedal robot to locomote in the presence of such disturbances.
This motivates the need for analytical tools aimed at generalized notions of stability -- robustness. 
Towards this, we propose a novel definition of robustness, termed \emph{$\delta$-robustness}, to characterize the domain on which a nominal periodic orbit remains stable despite uncertain terrain. This definition is derived by treating perturbations in ground height as disturbances in the context of the input-to-state-stability (ISS) of the extended Poincar\'{e} map associated with a periodic orbit. The main theoretic result is the formulation of robust Lyapunov functions that certify $\delta$-robustness of periodic orbits.  This yields an optimization framework for verifying $\delta$-robustness, which is demonstrated in simulation with a bipedal robot walking on uneven terrain.


\end{abstract}

\section{Introduction}

Achieving stable bipedal locomotion is a challenging control task---especially when locomoting on rough terrain. 
One approach with demonstrated success is to generate nominal walking behaviors, encoded by periodic orbits, and then use either feedback controllers or online planning to drive the system to these nominal behaviors \cite{grizzle2014models, griffin2017nonholonomic}. A benefit of this approach is that the stability of the nominal gait can then be analyzed using the method of Poincar\'e sections for systems with impulse effects \cite{morris2005restricted, morris2009hybrid}, i.e., one need only check the eigenvalues of the Poincar\'e map. 
Yet this notion of stability is inherently local and does not \new{provide provable guarantees of stability in the presence of disturbances such as those experienced with varying terrain height.}  




There have been approaches that have aimed to analyze the robustness of bipedal walking.  
Examples include the gait sensitivity norm \cite{hobbelen2007disturbance}, and the transverse linearization \cite{manchester2011stable}.  Yet these tools do not provide theoretical certificates of robustness. Similarly, existing work has synthesized bipedal walking gaits that are maximally robust to known environmental disturbances \cite{dai2012optimizing,park2012finite,hamed2016exponentially,tucker2022robust}.  While these have worked well in practice, again there is a lack of theoretic tools to formally asses their robustness, i.e., characterizing the domain on which behaviors are stable.  As a step in this direction, input-to-state stability (ISS)  \cite{sontag2008input} has been effectively leveraged in the context of robotic walking and running for uncertain dynamics \cite{kolathaya2018input, ma2017bipedal}. \new{However, this previous work limits the class of disturbances, $d$, that can be handled to those captured in a control-affine form (i.e., $\dot{x} = f(x) + g(x)u(x) + g(x)d$).}

\new{In contrast, our work formulates a notion of robust walking that quantifies the gap between stability and robustness mathematically by explicitly considering disturbances to the guard condition (commonly selected to be the ground height).}
\new{By considering this non-affine class of disturbances, our work is able to  define what it means for a periodic orbit to be certifiably robust to uncertain terrain as illustrated in Fig. \ref{fig: uncertainguard}. 
Specifically, we define the $\delta$-robustness of periodic orbits as the maximum disturbance in the guard condition that can be accommodated while remaining stable to a neighborhood.}
The main result of our paper is the formulation of robust Lyapunov functions that certify the robustness of periodic orbits to disturbances in the environment. 
The leads to an algorithm for certifying the $\delta$-robustness of walking gaits, as demonstrated in simulation with a seven-link bipedal robot walking on uneven terrain.



\begin{figure}[tb]
    \centering
    \includegraphics[width=0.99\linewidth]{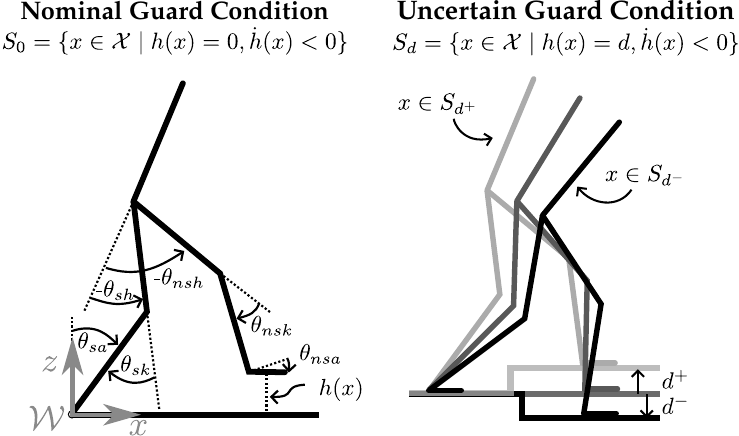}
    \vspace{-7mm}
    \caption{A depiction of (left) the configuration coordinates for a seven-link walker and (right) the uncertain guard condition. }
    \label{fig: uncertainguard}
    \vspace{-4mm}
\end{figure}

\section{Preliminaries}
\label{sec: preliminaries}
Walking naturally lends itself to be modeled as a hybrid system because of the presence of both continuous dynamics (during the swing phase) and discrete dynamics (at swing foot impacts) \cite{westervelt2018feedback}. Additionally, the dynamics of walking can be separated into those that can be controlled using actuation, and those that are uncontrollable -- termed the zero dynamics. 

\newsec{Hybrid Systems}
Consider a hybrid control system with states $x \in \X \subset \R^n$ and a control input $u \in \U \subset \R^m$. Given a continuously differentiable function\new{\footnote{\new{Note that $h$ must be selected such that it does not lie within the null space of the actuation matrix, i.e., $L_gh(x)\not=0$}}} $h:\mathcal{X}\to\R$, 
let $D \subset \X$ denote the admissible domain on which the continuous-time dynamics evolve and $S \subset D$ denote the \textit{guard} (also commonly called the \textit{switching surface}), defined as:
\begin{align}
    D &= \{x\in\mathcal{X} \mid h(x) \ge 0\}, \\
    S &= \{x\in \mathcal{X} \mid h(x) = 0,~\dot h(x) < 0\}.
    \label{eq: zeroguard}
\end{align}
For states $x^-\in S$, a discrete impact map $\Delta:S \to D$, termed the \textit{reset map} is applied. Thus, the complete hybrid system can be modeled as:
\begin{numcases}{\mathcal{H}\mathcal{C}  = }
\dot{x} = f(x) + g(x) u & $x \in D \setminus S$, \label{eq: continuouscontrol}
\\
x^+ = \Delta(x^-) & $x^- \in S$, \label{eq: discretecontrol}
\end{numcases}
where \eqref{eq: continuouscontrol} and \eqref{eq: discretecontrol} denote the continuous-time and discrete-time dynamics respectively.  It is assumed (as is typical) that all quantities in $\mathcal{H}\mathcal{C}$ are locally Lipshitz continuous, e.g., the impact map $\Delta$ is locally Lipschitz.
This follows from the assumption of perfectly plastic impacts \cite{glocker1992dynamical}. 
    Importantly, note that for impact maps based on rigid-body contacts \cite{hurmuzlu1994rigid}, the impact map does not depend on the ground height. 



Given a locally Lipschitz feedback controller $u = k(x)$, the result of applying this to the hybrid control system results in a hybrid system: 
\begin{numcases}{\mathcal{H}  = }
\dot{x} = f_{\rm cl}(x) := f(x) + g(x) k(x) & $x \in D \setminus S$, \label{eq: continuous}
\\
x^+ = \Delta(x^-) & $x^- \in S$, \label{eq: discrete}
\end{numcases}
The local Lipschitz continuity of the continuous dynamics \eqref{eq: continuous} implies that solutions exist and are unique locally.  
We will use the flow notation for these solutions, $\varphi_t(x_0)$, which is the solution to the continuous dynamics at time $t\in \R_{\geq 0}$ with initial condition $x_0 \in D$. 
Under the assumption of non-Zenoness, the flow of the hybrid system is given by: 
$$
\varphi_t(x_0) = \varphi_{t-\tau_k}(x_k^+) , \qquad t \in [\tau_k,\tau_{k+1})
$$
where $\tau_k$ are the ``impact'' times and $x_k^+$ the post-impact states, determined by the consistency conditions:
\begin{eqnarray}
\label{eqn:consistency}
x_k^+ = \Delta(x^-_k), \qquad x_k^- = \varphi_{\tau_k-\tau_{k-1}}(x^+_{k-1}) \in S,
\end{eqnarray}
for $k \geq 1$, with $\tau_0 = 0$ and $x_0 \in D$ the initial condition.  When $x_0 \in S$ one trivially takes $x^-_1 = x_0$ and $\tau_1 = \tau_0$.

\newsec{Periodicity of Hybrid Systems}
The flow $\varphi_t(x_0)$ of \eqref{eq: continuous} is periodic with period $T \in \R_{\geq 0}$ if there exists a point $x^* \in S$ satisfying $\varphi_T(\Delta(x^*)) = x^*$. The periodic orbit associated with this periodic flow is denoted:
\begin{align}
    \O := \{\varphi_t(\Delta(x^*)) \in D \mid 0 \leq t \leq T_I(x^*) = T\},
\end{align}
with $T_I: \widetilde{S} \to \R$ being the time-to-impact function:
\begin{align}
\label{eqn:timetoimpact}
    T_I(x) = \inf\{t \geq 0 \mid \varphi_t(\Delta(x)) \in S\}.
\end{align}
As proven in Lemma 3 of \cite{grizzle2001asymptotically}, the time-to-impact function is continuous at points $x \in \widetilde{S}$ satisfying the conditions $\widetilde{S} := \{ x \in S \mid 0 < T_I(x) < \infty\}$. Thus, $T_I$ is well-defined for $\widetilde{S}$.  The periodic orbit, $\O$, is exponentially stable if it is exponentially stable as a set: for $x_0 \in D$:
$$
\| \varphi_t(x_0) \|_{\O} \leq 
M e^{-\alpha t} \| x_0 \|_{\O}
$$
where $\| x \|_{\O}= \inf_{y \in \O} \| x - y \|$ is the set distance. 


The exponential stability of this periodic orbit $\O$ can be analyzed via the Poincar\'e map.  In particular, $S$ is a Poincar\'e section (and well-defined as such due to the assumption that $\dot{h}(x) < 0$), and associated with this Poincar\'e section is the Poincar\'e map $P: \widetilde{S} \to S$ defined as:
\begin{align}
    P(x^-) := \varphi_{T_I(x^-)}\left(\Delta(x^-)\right).
    \label{eq: poincare}
\end{align}
The Poincar\'e map describes the evolution of the hybrid system as a discrete-time system:
\begin{align}
    x^-_{k+1} = P(x^-_k), ~k=0,1,\dots,
    \label{eq: discrete-system}
\end{align}
wherein $x_k^-$ is just given as in \eqref{eqn:consistency}. 
In \cite{morris2005restricted} (see also \cite{nersesov2002generalization}, Theorem 2.1), it was proven that a periodic orbit $\O$ is exponentially stable if and only if $x^* \in \O \cap S$ is an exponentially stable fixed point of the discrete-time system \eqref{eq: discrete-system}. This is summarized in the following:


\begin{theorem}[\cite{morris2005restricted}]
\label{thm:conttoPexpstability}
A periodic orbit $\O$ is exponentially stable if and only if for the corresponding fixed point $P(x^*) = x^* \in S$, there exist $M > 0$, $\alpha \in (0,1)$, and some $\delta > 0$ such that:
\begin{align}
  \forall ~ x \in B_{\delta}(x^*) \cap \widetilde{S} &\quad \implies \quad  \nonumber\\
 & \| P^i(x) - P(x^*)\| \leq M\alpha^i \|x - x^*\|, \nonumber
\end{align}
with $P^i(x)$ denoting the Poincar\'e map applied $i \in \N_{\geq 0}  = \{0,1,\dots,n,\dots\}$ times.
\label{def: expstability}
\end{theorem}

\section{An ISS Perspective on Walking: $\delta$-Robustness}
\label{sec: uncertainguard}

This section provides the key formulation of robustness considered throughout this paper---that of $\delta$-robustness.  The core concept behind this definition is stability in and of itself is not a sufficiently rich concept to capture robustness, since it is purely local.  Thus, we define a notion of robustness leveraging the extended Poincar\'e map \new{(which extends the Poincar\'e map to consider general guard conditions)} and input-to-state stability, wherein the inputs are the disturbances associated with uncertain guard conditions.

\newsec{Motivation}
Practically, the stability of periodic orbits can be analyzed by evaluating the eigenvalues of the Poincar\'e return map linearized around the fixed point. Specifically, if the magnitude of the eigenvalues of $DP(x^*) = \frac{\partial P}{\partial x}(x^*)$ is less than one (i.e. $\max |\lambda (DP(x^*))| < 1$), then the fixed point is stable \cite{morris2005restricted,perko2013differential}. 
%
\new{While this property implies that the Poincar\'e map is robust to sufficiently small perturbations, it is often incorrectly assumed that the magnitude of the eigenvalues say something deeper about the broader robustness of the periodic orbit to perturbations.}
This is not the case, as the following example illustrates.


\begin{example}
Consider a seven-link bipedal robot as shown in Figure \ref{fig: uncertainguard}.  To illustrate how the eigenvalues associated with the linearization fail to tell the whole story, we will consider the robustness of two gaits to differing ground height conditions.  
As illustrated in Figure \ref{fig: motivation}, the classic Poincar\'e analysis does not accurately reflect the robustness of periodic orbits to local disturbances in the guard condition. That is, the gait with the smaller maximum eigenvalue (magnitude) is more fragile to changing ground heights. 
\end{example}

\begin{figure}[tb]
    \centering
    \includegraphics[width=\linewidth]{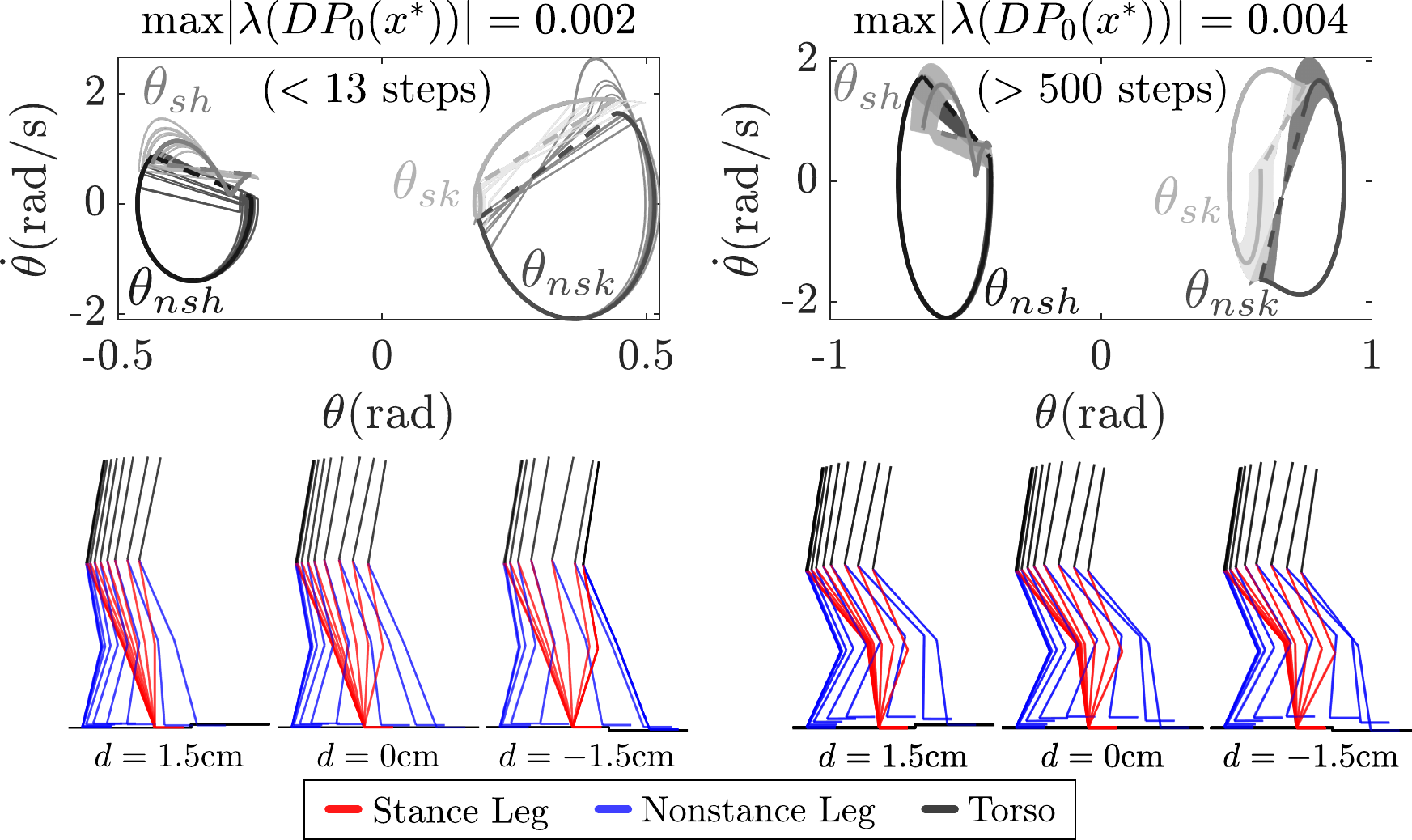}
    \vspace{-0.6cm}
    \caption{The phase portraits at the top of the figure illustrate the walking for uncertain guard conditions $S_{d_k}$ with $d_k \sim U(-\delta,\delta)$ (in this example, $\delta = 1.5$cm) for $k = 500$ steps. Visualizations of the walking gaits for three step conditions are provided at the bottom. The results demonstrate that a periodic orbit with $\max |\lambda(D P_0(x^*))| < 1$ (on the left) is not robust to variations in the guard condition (the orbit diverged after only 13 steps), while a periodic orbit with a larger $|\lambda|$ (on the right) is comparatively more robust. This motivates the need for an ISS perspective.}
    \label{fig: motivation}
    \vspace{-0.4cm}
\end{figure}

\newsec{Uncertain Guard Conditions}
\new{To formulate a notion of robustness, uncertain guard conditions are considered---this, for example, captures uncertain ground height for walking robots. Specifically, as done in \cite{hamed2016exponentially}, the Poincar\'e map can be extended to explicitly consider changes to the guard condition (i.e., $h(x) = d$). First, define a general guard as:}
\begin{align}
    S_d &= \{x\in \mathcal{X} \mid h(x) = d, ~\dot h(x) < 0\},
    \label{eq: heightguard}
\end{align}
with $d \in \mathbb{D}$ and $\mathbb{D} := [d^-,d^+] \subset \R$ for some $d^- < 0 < d^+$. Using this general guard definition, the previous guard \eqref{eq: zeroguard} is now denoted as $S_0$.  Under the assumption that $S_d \subset D$ for all $d \in \mathbb{D}$, we have a corresponding hybrid system: 
\begin{numcases}{\mathcal{H}_d  = }
\dot{x} = f_{\rm cl}(x) := f(x) + g(x) k(x) & $x \in D \setminus S_d$, \label{eq: continuousd}
\\
x^+ = \Delta(x^-) & $x^- \in S_d$, \label{eq: discreted}
\end{numcases}


\new{Next, we must modify the time-to-impact function to be defined on a neighborhood of the fixed point $x^*$. In particular, the time-to-impact function exists as a result of the implicit function theorem \cite{lee2010introduction} applied to the implicit function (of time) $ h(\varphi_t(\Delta(x)))$ which therefore satisfies: 
$h(\varphi_T(\Delta(x^*))) = 0$, and $\dot{h}(\varphi_T(\Delta(x^*))) < 0$,
for $x^* \in \O \cap S$. Thus, there exists an explicit function $T_e: B_{\rho}(x^*) \subset D \to \R$, for some $\rho > 0$\footnote{We assume throughout the paper that for all $\rho > 0$ of interest, the domain $D$ of the continuous dynamics is appropriately chosen so that $B_{\rho}(x^*) \subset D$.}, termed the \emph{extended time-to-impact function} satisfying:
\begin{eqnarray}
\label{eqn:extendedtimetoimpact}
h(\varphi_{T_e(x)}(\Delta(x))) = 0, \qquad \forall ~ x \in B_{\rho}(x^*).
\end{eqnarray}
It follows that $T_I$ in \eqref{eqn:timetoimpact} is just $T_I = T_e|_{S}$, wherein the Poincar\'e map is given by considering only $x \in B_{\rho}(x^*) \cap S$.}  
This function can be further extended (as a partial function) to account for varying guards: $T_e : B_{\rho}(x^*) \times \mathbb{D} \partialto \R$:
\begin{align}
    T_e(x_0,d) := \inf\{ t \geq 0 \mid \varphi_t(\Delta(x_0)) \in S_d \}.
\end{align}
Importantly, this is a partial function because (by the implicit function theorem) it is only well-defined for $d = 0$ and by continuity sufficiently small $d^-$ and $d^+$.  
Using this extended time-to-impact function, we can define the \textit{extended Poincar\'e map} as a partial function: $P_d: B_{\rho}(x^*) \partialto S_d$:
\begin{align}
\label{eqn:Pd}
    P_{d}(x^-) := \varphi_{T_e(x^-,d)}(\Delta(x^-)).
\end{align}
This allows us to frame walking with uncertain guards as a discrete-time control system.







\newsec{Connections with Input-to-State Stability}
It is important to note that we can view \eqref{eqn:Pd} as a dynamical system evolving with an ``input'' given by the guard height: $d = h(x)$. 
In particular, this leads to the discrete-time dynamical system: 
\begin{align}
    \label{eqn:discretetimeP}
    x_{k+1} = \P(x_k,d_k) := P_{d_k}(x_k), 
\end{align}
for some sequence of $d_k \in [d^-,d^+] \subset \R$, $k \in \N_{\geq 0}$, determining the guard height specific to step $k \in \N_{\geq 0}$ such that $x_{k+1} \in S_{d_k}$.  The result is a partial function:
$$
\P : B_{\rho}(x^*) \times  [d^-,d^+] \partialto S_{[d^-,d^+]} : = \bigcup_{d \in [d^-,d^+]} S_d,
$$
wherein we assume that $B_{\rho}(x^*) \subset S_{[d^-,d^+]}$ (or a smaller $\rho$ is chosen so that this holds). The partial function nature of $\P$ implies that solutions may not exist for all time, i.e., the solution $x_k$ might leave the ball $B_{\rho}(x^*)$ on which $\P$ is well-defined.


Given the discrete-time system \eqref{eqn:discretetimeP}, and the fact that we view the input $d$ as a disturbance, there are obvious connections with input-to-state stability \cite{jiang2001input}.  In our setting, the discrete-time system $x_{k+1} = \P(x_k,d_k)$ (with $d_k$ viewed as an input) is \emph{input-to-state stable (ISS)} if: 
\begin{eqnarray}
\label{eqn:expiss}
\| x_k - x^* \| \leq   \beta(\| x_0 - x^*\|,k)  + \gamma( \| d \|_{\infty}  )
\end{eqnarray}
for $k \in \N_{\geq 0}$, $\beta$ a class $\KL$ function, and $\gamma$ a class $\K$ function.  Note that here $\| d \|_{\infty}  = \max\{-d^-,d^+\}  $ since $d : \N_{\geq 0} \to [d^-,d^+]$ is scalar valued and takes values in an interval.  Also note that, in the context of locomotion, we are especially interested in exponential stability.  To certify exponential ISS, the class $\KL$ function becomes: $\beta(r,k) = M\alpha^k r$ for $M > 0$ and $\alpha \in (0,1)$.  The end result is the exponential ISS (E-ISS) condition: 
\begin{eqnarray}
\label{eqn:expiss}
\| x_k - x^* \| \leq M\alpha^k \| x_0 - x^*\|  + \gamma( \max\{-d^-,d^+\} )
\end{eqnarray}
This allows us to formulate a notion of robustness. 

\newsec{$\bm{\delta}$-Robustness}  We now have the necessary components to present the key concept of this paper: $\delta$-robustness.  The goal in formulating this notion of robustness is to find a single scalar constant, $\delta \geq 0 $, that characterizes the robustness of a periodic orbit $\O$ in the context of uncertain guard height.  In this context, we wish to leverage \eqref{eqn:expiss}---yet the class $\K$ function $\gamma$ gives a degree of freedom that is undesirable in designing a metric for robustness.  This observation leads to:

\begin{definition}
The periodic orbit $\O$ is \textbf{$\bm{\delta}$-robust} for a given $\delta > 0$ if for the discrete-time dynamical system in \eqref{eqn:discretetimeP} with $d^- = -\delta$ and $d^+ = \delta$, that is: 
\begin{align}
 \P : & B_{\rho}(x^*) \times [-\delta,\delta] \to S_{[-\delta,\delta]} \nonumber\\
\label{eqn:Pdelta}
& x_{k+1} = \P(x_k,d_k), \qquad   d_k \in [-\delta,\delta],
\end{align}
there exists a forward invariant set $W \subset B_{\rho}(x^*)$ and for all $x_0 \in W$: 
\begin{eqnarray}
\label{eqn:deltarobustness}
\| x_k - x^* \| \leq M\alpha^k \| x_0 - x^* \|  + \gamma \delta, \qquad \forall k \in \N_{\geq 0}, 
\end{eqnarray}
for some $\gamma > 0$, $M > 0$, and $\alpha \in (0,1)$.
The periodic orbit is \textbf{robust} if it is $\delta$-robust for some $\delta > 0$, and the largest scalar $\overline{\delta}$ such that $\O$ is $\overline{\delta}$-robust is the \textbf{robustness} of $\O$. 
\end{definition}

This seemingly simple definition encodes a surprising amount of information.  First, the forward invariance of $W \subset B_{\rho}(x^*)$ implies that $\P : B_{\rho}(x^*) \times [-\delta,\delta] \to S_{[-\delta,\delta]}$ is a function (rather than a partial function) when restricted to the set $W$. Additionally, the actual $\delta$-robustness condition \eqref{eqn:deltarobustness} is an ISS condition, albeit slightly stronger to remove the dependence on the class $\K$ function and replace this with the constant $\gamma$. Even so, the connections with ISS are important since the associated machinery can be leveraged.

To provide an example of how ISS can inform our thinking on $\delta$-robustness, consider the case when $\O$ is exponentially stable, i.e., $x_{k+1} = \P(x_k,0)$ has an exponentially stable fixed point: $x^* = P_0(x^*)$, i.e., the 0-input system is exponentially stable.  There are no guarantees that $\O$ is thus $\delta$-robust (see \cite{jiang2001input} where a counter example shows that \new{given arbitrarily bounded disturbances, then local asymptotic stability is not enough to guarantee ISS}).  That is, stability does not imply robustness.  

\begin{example}
Returning to the example of the seven-link walker, we can heuristically calculate the $\delta$-robustness associated with the two gaits. 
Specifically, Fig. \ref{fig: exampleissbound} illustrates the ISS-perspective of $\delta$-robustness for the orbits first illustrated in Fig. \ref{fig: motivation}. As shown, the orbit that was robust in Fig. \ref{fig: motivation} satisfies the condition that $W \subset B_{\rho}(x^*)$ is forward invariant ($\delta = 1.5$cm in this example), and $\|x_k - x^*\|$ remains bounded for $\gamma = 36.8$. Comparatively, the orbit that was not robust in Fig. \ref{fig: motivation} experienced a pre-impact state that was outside of $B_{\rho}(x^*)$ and therefore $W$ was not forward invariant.
\end{example}

\begin{figure}[tb]
    \centering
    \includegraphics[width=\linewidth]{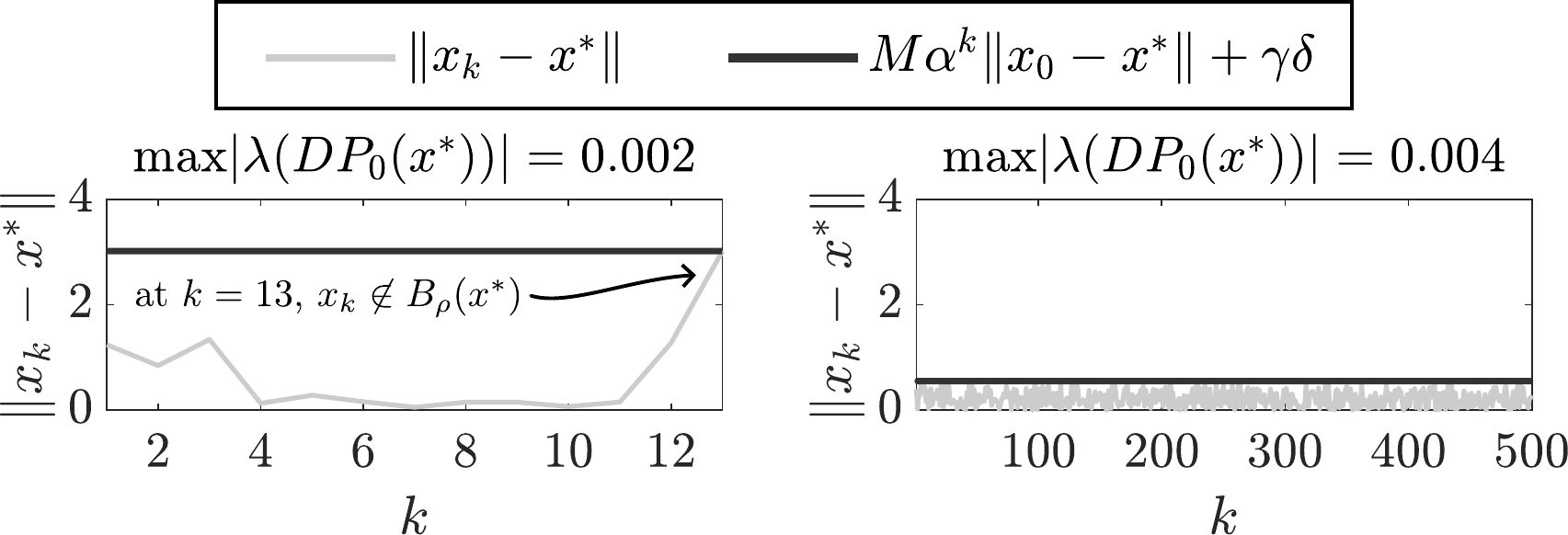}
    \vspace{-0.6cm}
    \caption{
    On the left, the non-robust periodic orbit (as illustrated on the left of Fig. \ref{fig: motivation}) does not satisfy the conditions for $\delta$-robustness for $\delta = 0.015$ (specifically, there does not exist a forward invariant set $W$). In comparison, the robust orbit (as illustrated on the right of Fig. \ref{fig: motivation}) satisfies the definition of $\delta$-robustness with $\gamma = 36.8$ and $\delta = 0.015$m.}
    \label{fig: exampleissbound}
    \vspace{-0.4cm}
\end{figure}

\section{Lyapunov Conditions for $\delta$-robustness}
\label{sec:main}

In this section we present the main theoretic result: Lyapunov conditions for the $\delta$-robustness of periodic orbits.  These conditions, and constructions, follow naturally from the ISS perspective employed in defining $\delta$-robustness.  But care is needed given the complexity of the Poincar\'{e} map.  Importantly, these conditions will lead to an approach for the verification of $\delta$-robustness, as presented in the next section.

\begin{definition}
Consider the discrete-time  dynamical system in \eqref{eqn:Pdelta}.  A function $V : B_{\rho}(x^*) \to \R_{\geq 0}$, for $B_{\rho}(x^*)$ as in \eqref{eqn:extendedtimetoimpact}, is a \textbf{robust Lyapunov function} if:  
\begin{align}
\label{eqn:lyap1}
       k_1 & \| x - x^* \|^c   \leq  V(x)   \leq k_2 \|   x - x^* \|^c  \\
\label{eqn:lyapimplication}
   \| x  - x^* \|  \geq  \chi d & \quad \implies \quad \\
  \Delta V(x,d) &  :=  V(\P(x,d)) - V(x)  \leq - k_3 \| x - x^* \|^c \nonumber
\end{align}
for $\chi, k_1,k_2,k_3, c > 0$ and all $x \in B_{\rho}(x^*)$.  
\end{definition}

\begin{remark}
Note that \eqref{eqn:lyapimplication} can be equivalently restated as: 
\begin{eqnarray}
   V(\P(x,d)) - V(x)    \leq - k_4 \|  x - x^*  \|^c  + \frac{1}{2} \sigma |d|^c,
  \label{eqn:lyap2}
   \end{eqnarray}
 where $\sigma > 0$.  In particular, the corresponding quantities are related via: $k_3 = \frac{1}{2} k_4$ and $\chi = k_4^{-\frac{1}{c}} \sigma^{\frac{1}{c}}$. 
\end{remark}

\newsec{Main result}  We can now state the main result of the paper.  To do so, recall that a 
\emph{Lyapunov sublevel set} is given by: 
\begin{eqnarray}
\Omega_{r} = \{ x \in \R^n ~ | ~
V(x) \leq r\}. 
\end{eqnarray}
This will be essential in establishing:

\begin{theorem}
\label{thm:main}
Consider the discrete-time dynamical system $x_{k+1} = \P(x_k,d_k)$ in \eqref{eqn:Pdelta} with associated periodic orbit $\O$.  If there exists a robust Lyapunov function, $V : B_{\rho}(x^*) \to \R_{\geq 0}$, and:
\begin{eqnarray}
\label{eqn:deltabound}
\delta < \delta_{\max} := \left(\frac{k_1}{\chi^c k_2}\right)^{\frac{1}{c}} \rho,
\end{eqnarray}
then the periodic orbit $\O$ is $\delta$-robust with: 
\begin{align}
 & \qquad   W =   \Omega_{r(\delta)},  \quad \mathrm{for} \quad r(\delta) := k_2 (\chi \delta) ^c  \\
\gamma = &\left( \frac{k_2}{k_1} \right)^{\frac{1}{c}} \chi,   \quad 
M= \left(\frac{k_2}{k_1}\right)^{\frac{1}{c}}, \quad 
\alpha = \left( 1 - \frac{k_3}{k_2}\right)^{\frac{1}{c}}. \nonumber
\end{align}
\end{theorem}

This theorem is, overall, a variation on Lemma 3.5 in \cite{jiang2001input}.  The proof here follows a similar overall arc, although there are key differences made necessary by the fact that $\P$ is only a partial function. 
This motivates the first Lemma. 

\begin{lemma}
The function $\P: B_{\rho}(x^*) \times [-\delta,\delta] \to S_{[-\delta,\delta]}$ given in \eqref{eqn:Pdelta} is well-defined for all $x \in  B_{\rho}(x^*) $, i.e., for all $x \in B_{\rho}(x^*) $, $\P(x,d)$ exists and satisfies $\P(x,d) \in S_{[-\delta,\delta]}$.
\end{lemma}

\begin{proof}
By the construction of the extended Poincar\'e map, $P_0$ is well-defined on $B_{\rho}(x^*)$, i.e., for all $x \in B_{\rho}(x^*)$ it follows that $\P(x,0) \in S_0$, i.e., $h(\varphi_{T_e(x,0)}(\Delta(x))) = 0$.  Therefore:
$$
h(\varphi_{t}(\Delta(x))) = \int_{T_e(x,0)}^t 
\dot{h}(\varphi_{\tau}(\Delta(x))) d \tau.
$$
But $\dot{h}(x) < 0$ for all $x \in  S_{[-\delta,\delta]}$ by definition.  Therefore, on the closed set defined by $- \delta \leq h(x) \leq \delta$, $\dot{h}$ takes a minimum and maximum value: $\underline{h} < \overline{h} < 0$.  This implies that: 
$$
\underline{h} (t - T_e(x,0))  \leq h(\varphi_{t}(\Delta(x))) \leq \overline{h}  (t - T_e(x,0)).
$$
Thus, there exists a $t$ (possibly negative) such that $h(\varphi_{t}(\Delta(x))) = d$. This $t = T_e(x,d)$. 
\end{proof}

Since $\P$ is well-defined, we can now find a set such that $x_{k+1} = \P(x_k,d_k)$ is defined for all $k$, i.e., a forward invariant set contained in $B_{\rho}(x^*)$, using Lyapunov sublevel sets.

\begin{lemma}
\label{lem:levelset}
If $\delta < \delta_{\max}$, with $\delta_{\max}$ in \eqref{eqn:deltabound}, 
then for $r(\delta) := k_2 (\chi \delta) ^c$ it follows that:
$$
B_{\chi \delta}(x^*) \subset \Omega_{r(\delta)} \subset B_{\rho}(x^*).
$$
Moreover, the set $\Omega_{r(\delta)}$ is forward invariant. 
\end{lemma}




\begin{proof}
For $x \in B_{\chi \delta}(x^*)$: 
$$
\| x - x^* \| < \chi \delta ~  \Rightarrow  ~ 
V(x) \leq k_2 \| x - x^* \|^c < k_2 (\chi \delta)^c = r(\delta)
$$
and therefore $B_{\chi \delta}(x^*) \subset \Omega_{r(\delta)}$.  Now if $r(\delta) < k_1 \rho^c $ (which is equivalent to the condition \eqref{eqn:deltabound}) it follows that:
$$
V(x) \leq r(\delta) \quad \Rightarrow \quad 
k_1 \| x - x^* \|^c \leq V(x)  \leq r(\delta) <  k_1 \rho^c 
$$
And therefore: $\Omega_{r(\delta)} \subset B_{\rho}(x^*)$.   
%
%
Finally, since for $\delta < \delta_{\max}$ we have $B_{\chi \delta}(x^*) \subset \Omega_{r(\delta)}$, it follows that on the boundary of $\Omega_{r(\delta)}$, namely $\partial \Omega_{r(\delta)}$, condition \eqref{eqn:lyapimplication} is active and therefore: $\Delta V(x,d) < 0$.  The forward invariance of $\Omega_{r(\delta)}$ follows. 
\end{proof}

Lemma \ref{lem:levelset} gives an upper bound on the $\delta$-robustness of a given periodic orbit $\O$, namely $\delta_{\max}$, based upon the domain of definition of $\P$.  It also establishes the forward invariance of  $\Omega_{r(\delta)}$.  
Leveraging this, we can prove the main result. 

\begin{proof}[Proof of Theorem \ref{thm:main}]
Let $x_0 \in \Omega_{r(\delta)}$, wherein the forward invariance of $\Omega_{r(\delta)}$ (Lemma \ref{lem:levelset}) implies $x_k \in \Omega_{r(\delta)} \subset B_{\rho}(x^*)$ for all $k \in \N_{\geq 0}$.  Thus both $\P$ and $V$ are well-defined.
We consider two cases: $x_0 \notin B_{\chi \delta}(x^*)$ and $x_0 \in B_{\chi \delta}(x^*)$.

\vspace{0.1cm}
\underline{$\| x_0 - x^* \| \geq \chi \delta$:}  In this case the implication \eqref{eqn:lyapimplication} is active: 
$$
\Delta V \leq - \frac{k_3}{k_2} V \quad \implies \quad
V(x_k) \leq \left( 1 - \frac{k_3}{k_2}\right)^k V(x_0)
$$
where the implication follows from applying the inequality on the right recursively (see also the comparison lemma \cite{jiang2002converse}).  Therefore, using the inequalities in \eqref{eqn:lyap1} we have: 
\begin{eqnarray}
\label{eqn:Malpha}
\| x_k - x^* \| \leq  \underbrace{\left( \frac{k_2}{k_1} \right)^{\frac{1}{c}}}_{M} \underbrace{\left( 1 - \frac{k_3}{k_2}\right)^{\frac{k}{c}}}_{\alpha^k} \| x_0 - x^* \|.
\end{eqnarray}
Finally, note that $k_3/k_2 < 1$ as otherwise $V(x_k)$ would be negative for $k = 1$ which is impossible.  Therefore, $\alpha < 1$. 

\vspace{0.1cm}
\underline{$\| x_0 - x^* \| < \chi \delta$:}  While the implication in \eqref{eqn:lyapimplication} no longer holds, we still have $x_k \in \Omega_{r(\delta)}$.  As a result: 
\begin{align}
k_1 \| x_k - x^* \|^c \leq  V(k_k) \leq  r(\delta) =  & k_2 (\chi \delta) ^c 
  \nonumber\\
  \label{eqn:gammaeqn}
 \quad \implies \quad \| x_k - x^* \| \leq &  \underbrace{\left( \frac{k_2}{k_1} \right)^{\frac{1}{c}} \chi}_{\gamma} \delta 
\end{align}
Therefore, for $M$, $\alpha$ in \eqref{eqn:Malpha} and $\gamma$ in \eqref{eqn:gammaeqn} we have:
\begin{eqnarray}
\| x_k - x^* \| & \leq  & \max \{ M \alpha^k \| x_0 - x^* \|, \gamma \delta \} \nonumber\\
& \leq &  M \alpha^k \| x_0 - x^* \| + \gamma \delta \nonumber
\end{eqnarray}
as desired, i.e., $\delta$-robustness is established with $W =  \Omega_{r(\delta)}$ the required forward invariant set.  
\end{proof}

\section{Algorithmic Verification of $\delta$-Robustness 
}

\begin{figure}[tb]
    \centering
    \includegraphics[width=\linewidth]{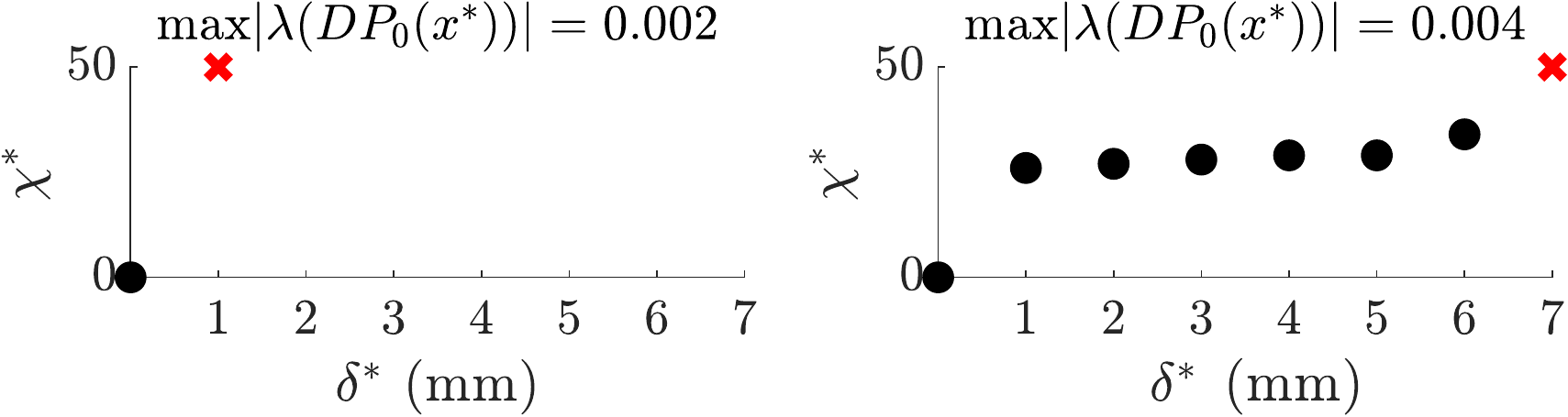}
    \vspace{-0.6cm}
    \caption{Results of the algorithmic approach to Opt. \eqref{eq: optdelta} for the gaits shown in Fig. \ref{fig: motivation} with the maximum allowable $\chi$ set to $50$. As shown, the gaits were determined to be $\delta$-robust for $\delta^* = 0$ and $\delta^*=6$mm, respectively.}
    \label{fig: algresults}
    \vspace{-0.5cm}
\end{figure}

\new{Finally, to verify the $\delta$-robustness of a given periodic orbit $\O$, we will synthesize an optimization framework that leverages the previously presented robust Lyapunov conditions.}

\newsec{Problem Setup}  
Assume the existence of a stable periodic orbit $\O$ and so $x_{k+1} = \P (x_k, 0)$ has an exponentially stable fixed point $x^*$.  For simplicity we will take $x^* = 0$ (achieved via the simple coordinate transformation $x \mapsto x - x^*$).  As a result, the linearization: 
$$
x_{k + 1} = A x_k  := D\P(0,0) x_k
$$
is exponentially stable. The Lyapunov matrix $P = P^T > 0$ is obtained by solving the discrete-time \new{Lyapunov} equation: 
$$
A^T P A - P = - Q
$$
for $Q = Q^T > 0$.   The end result is that the discrete-time Lyapunov function $V(x) = x^T P x$ satisfies:
\begin{align}
\label{eqn:lyap1lin}
    \lambda_{\min}(P) \| x  \|^2 \leq V(x)  & \leq \lambda_{\max}(P)  \|   x \|^2  \\
  V(A x) - V(x) &  \leq - \lambda_{\min}(Q) \|  x  \|^2. 
  \label{eqn:lyap2lin}
\end{align}
and thereby establishes exponential stability of the linear system (and the nonlinear system locally).  Unlike stability, it is not guaranteed that this Lyapunov function can be used to establish robustness.  Yet we will use it as a ``guess'' for a robust Lyapunov function in order to develop an algorithm to establish the robustness of a given gait $\O$. 

\newsec{Optimization Problem}  Recall that the invariant set used to establish $\delta$ robustness was defined in Lemma \ref{lem:levelset}, namely $\Omega_{r(\delta)}$.  In this case: 
$$
\Omega_{r(\delta)} = \{ x \in \R^n | V(x) = x^T P x \leq r(\delta) := k_2 (\chi \delta) ^c \}.
$$
Per the proof of Lemma \ref{lem:levelset} we therefore have:
$$
B_{r_1}(0) \subset \Omega_{r(\delta)}  \subset B_{r_2}(0), 
$$
with:
$$
r_1 := \chi \delta, \qquad  r_2 := \left(\frac{\lambda_{\max}(P)}{\lambda_{\min}(P)}\right)^{\frac{1}{2}} \chi \delta. 
$$
Then with the goal of finding the largest $\delta^* > 0$ such that $\O$ is $\delta$ robust, we formulate the following optimization problem: 
\begin{align}
    \label{eq: optdelta}
    (\delta^*,\chi^*) = \argmax_{\delta, \chi > 0} & ~ \delta \\
    \text{s.t. }  & V(\P(x,d)) - V(x)  \leq - k\|  x   \|^2  \notag \\
                  & \quad \forall ~ r_1  < \| x \| < r_2 , \quad \forall ~  d \in [-\delta,\delta], \notag
\end{align}
where $k \in (0,1)$ is a user-defined variable, and we take $Q = I$ (wherein $\lambda_{\min}(Q) = 1$) to remove decision variables.  

\begin{figure}[tb]
    \centering
    \includegraphics[width=\linewidth]{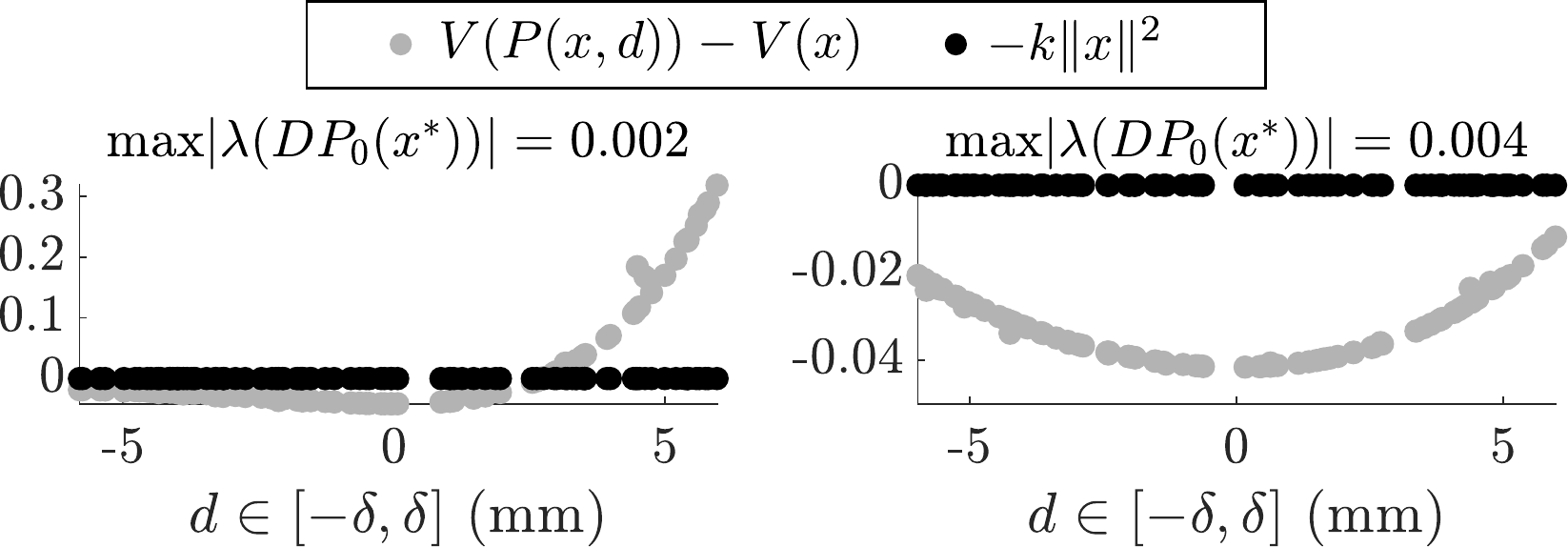}
    \vspace{-0.6cm}
    \caption{Illustration of the Lyapunov condition in \eqref{eq: optdelta} for 100 random samples $x \in B_{r_1}(0)$ with $d \sim U(-6\textrm{mm},6\textrm{mm})$. As shown, the Lyapunov condition is satisfied for the gait identified as being $\delta$-robust for $\delta = 6$mm with $\chi = 34$ (the corresponding ISS bound is illustrated in Fig. \ref{fig: issresults}).}
    \label{fig: lyapcondition}
    \vspace{-0.2cm}
\end{figure}

\begin{figure}[tb]
    \centering
    \includegraphics[width=\linewidth]{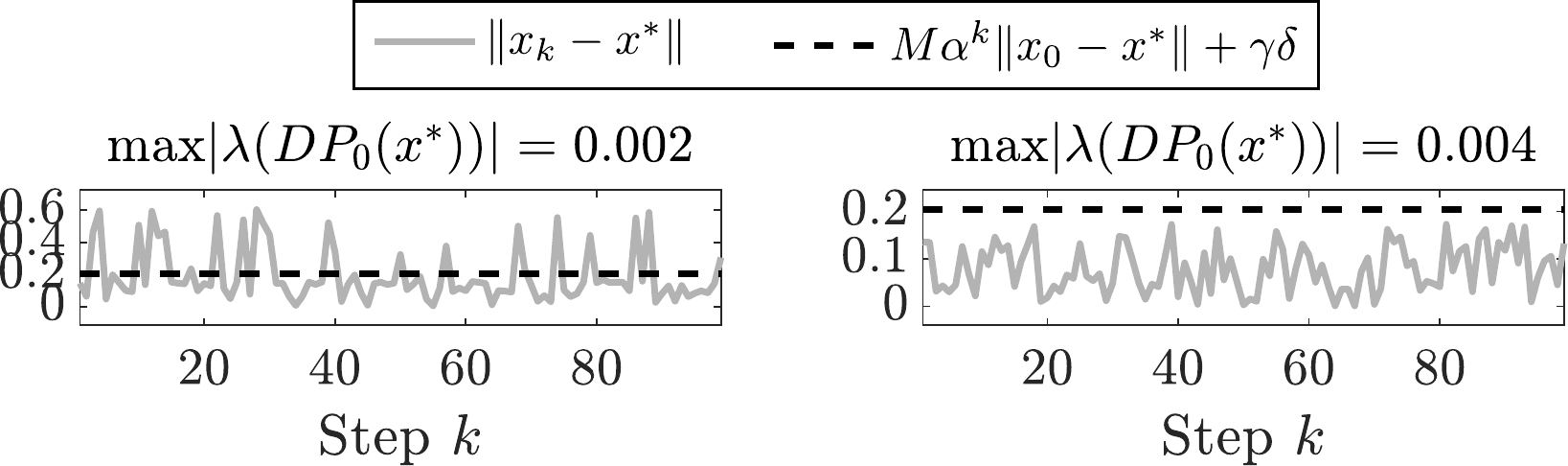}
    \vspace{-0.6cm}
    \caption{Verification of $\delta$-robustness for $\delta^* = 6$mm and $\chi^* = 34$ (selected based on the algorithm results shown in Fig. \ref{fig: algresults}). As shown in the figure, Gait 1 was not $\delta$-robust while Gait 2 was $\delta$-robust with $M$, $\gamma$, and $\alpha$ defined using the relationships derived in Theorem \ref{thm:main} and $V(x) = x^TPx$.}
    \label{fig: issresults}
    \vspace{-0.5cm}
\end{figure}

\new{
Since this optimization problem is bilinear and nonconvex, it is easier to approach algorithmically. Concretely, as outlined in Algorithm\footnote{The implementation of the algorithm, as well as its application towards evaluating the $\delta$-robustness of bipedal walking gaits, is provided in the repository: \url{https://github.com/maegant/deltaRobustness.git}} \ref{alg: optdelta}, this procedure consists of slowly increasing $\chi$ for each candidate $\delta$ and checking the Lyapunov condition in \eqref{eq: optdelta} for random samples $x \in B_{r_1}(0)$.
The advantage of this approach is that it is guaranteed to identify sets $\{\chi,\delta\}$ that certify $\delta$-robustness (assuming one exists and that $\Delta \chi$ is sufficiently small).} 
We demonstrate the algorithm for each of the two gaits illustrated in Fig. \ref{fig: motivation} with the results provided in Fig. \ref{fig: algresults}. As expected, the second gait illustrated in Fig. \ref{fig: motivation} and Fig. \ref{fig: exampleissbound} was verified to be $\delta$-robust, with $\delta^* = 6$mm. \new{Notably, this value is smaller than the 15mm ground heights empirically demonstrated in Fig. \ref{fig: motivation} due to the worst-case guarantees afforded by ISS.
} A visualization of the Lyapunov condition for 100 random samples ($x \in B_{r_1}(0)$) is provided in Fig. \ref{fig: lyapcondition} with the corresponding ISS bound in Fig. \ref{fig: issresults}.

\begin{figure}[tb]
\vspace{-5mm}
\centering
\begin{algorithm}[H]
\scriptsize
\caption{Algorithmic Approach to \eqref{eq: optdelta}}\label{alg: optdelta}
\begin{algorithmic}[1]
\STATE $\delta = 0$, $\chi_0 = 1$, $N$ = num. samples, $\{\Delta \delta, \Delta \chi, \chi_{\max}\} \in \R_{>0}$
\PROCEDURE{TestDelta}{$\delta$, $\chi_{\delta}$}
\FOR {$i = [1,\dots,N]$}
\STATE Sample $x' = x - x^*$ such that $\|x'\| = \chi_{\delta} \delta $
\IF {$V(\P(x',d)) - V(x') \leq -k\|x'\|^2, ~\forall d \in [-\delta,\delta]$}
    \STATE Repeat TestDelta($\delta + \Delta\delta$,$1$) 
\ELSE
    \IF {$\chi_{\delta} + \Delta \chi > \chi_{\max}$}
        \STATE Terminate with $\delta^* = \delta-\Delta\delta$, $\chi^* = \chi_{\delta^*}$
    \ELSE 
        \STATE Repeat TestDelta($\delta$,$\chi_{\delta} + \Delta\chi$)
    \ENDIF
\ENDIF
\ENDFOR
\end{algorithmic}
\end{algorithm}
\vspace{-10mm}
\end{figure}

\section{Conclusion}
In this work, a novel notion of robustness, \emph{$\delta$-robustness}, was formulated from the perspective of input-to-state stability.
Lyapunov conditions were also derived to certify $\delta$-robustness for a nominal periodic orbit. Future work includes directly evaluating $\delta$-robustness in the gait generation process to systematically generate periodic orbits that are robust to uncertain terrain. Additionally, sampling methods can be leveraged to obtain probabilistic guarantees on $\delta$-robustness. Lastly, the discrete-time Lyapunov condition can be translated to a stochastic condition in order to obtain more realistic (albeit probabilistic) estimates of the $\delta$-robustness.

 

\bibliographystyle{IEEEtran}
\bibliography{./Bibliography/IEEEabrv, ./Bibliography/References}
\addtolength{\textheight}{-3cm}   


\end{document}